\documentclass{article}

\usepackage{arxiv}

\usepackage[utf8]{inputenc} % allow utf-8 input
\usepackage[T1]{fontenc}    % use 8-bit T1 fonts
\usepackage{hyperref}       % hyperlinks
\usepackage{url}            % simple URL typesetting
\usepackage{booktabs}       % professional-quality tables
\usepackage{amsfonts}       % blackboard math symbols
\usepackage{nicefrac}       % compact symbols for 1/2, etc.
\usepackage{microtype}      % microtypography
\usepackage{lipsum}

\usepackage{natbib}
\usepackage[english]{babel}
\usepackage{graphicx}
\usepackage{algorithmic}

\usepackage{subfigure}
\usepackage{hyperref}

\usepackage{amssymb}
\usepackage{amsthm}
\usepackage{amsmath}
\usepackage{color}

\usepackage{algorithm}
\usepackage{enumitem}

\usepackage{wrapfig}

\newcommand{\bO}{{\mathcal{O}}}

\newcommand{\Pperp}{{{P_{u_1^{\perp}}}}}
\newcommand{\huip}{{\hat{u}_i^\perp}}

\newtheorem{theorem}{Theorem}[section]
\newtheorem{corollary}[theorem]{Corollary}

\newtheorem{lemma}[theorem]{Lemma}
\newtheorem{remark}[theorem]{Remark}
\newtheorem{assumption}{Assumption}[section]

\title{When Does Non-Orthogonal Tensor Decomposition Have No Spurious Local Minima?}

\author{
  Maziar Sanjabi
  \thanks{Data Science \& Operations Department, University of Southern, California, Los Angeles, CA 90089, USA}
  \\\texttt{maziar.sanjabi@gmail.com} \\
   \And
  Sina Baharlouei \thanks{Department of Industrial and Systems Engineering, University of Southern, California, Los Angeles, CA 90089, USA}\\
  \texttt{baharlou@usc.edu} \\
   \And
  Meisam Razaviyayn
  \footnotemark[2]\\
  \texttt{razaviya@usc.edu} \\
  \And
  Jason D. Lee
  \thanks{Department of Electrical Engineering, Princeton University, Princeton, NJ 08540, USA}\\
  \texttt{jasonlee@princeton.edu} \\
 }

\begin{document}
\maketitle

\begin{abstract}
    We study the optimization problem for decomposing $d$ dimensional fourth-order Tensors with $k$ non-orthogonal components. We derive \textit{deterministic} conditions under which such a problem does not have spurious local minima. In particular, we show that if $\kappa = \frac{\lambda_{max}}{\lambda_{min}} < \frac{5}{4}$, and incoherence coefficient is of the order $\bO(\frac{1}{\sqrt{d}})$, then all the local minima are globally optimal. Using standard techniques, these conditions could be easily transformed into conditions that would hold with high probability in high dimensions when the components are generated randomly. Finally, we prove that the tensor power method with deflation and restarts could efficiently extract all the components within a tolerance level $\bO(\kappa \sqrt{k\tau^3})$ that seems to be the noise floor of non-orthogonal tensor decomposition. 
\end{abstract}

\section{Introduction}
Tensor Decomposition approaches are demonstrated to be effective tools for modeling and solving a wide range of problems in the context of signal processing, statistical inference, and machine learning. In particular, many unsupervised learning problems such as Gaussian Mixture Models \citep{Gauss_Mix_ge_Kakade_2015}, Latent Dirichlet Allocation \citep{tensor_decomp_latent_anandkumar2014}, Topic Modeling \citep{cheng2015model, anandkumar2012two}, Hidden Markov Models \citep{RL_azizzadenesheli_2016}, Latent Graphical Models \citep{song2013hierarchical, graph_chaganty_Liang_2014}, and Community Detection \citep{al2017tensor, anandkumar2014tensor} can be modeled as a canonical decomposition (CANDECOMP) problem which is also known as Parallel Factorization (PARAFAC).

It has been shown that under mild assumptions such as $2R \leq k_A + k_B + k_C$ where $R$ is the number of components, and $k_A, k_B,$ and $k_C$ are the k-rank of the component matrices $A, B,$ and $C$ respectively, the CANDECOMP/PARAFAC (CP) decomposition exists and it is unique \citep{harshman1970foundations, kruskal1977three}. Finding such a decomposition is an NP-hard problem in general \citep{haastad1990tensor, hillar2013most}. Despite the hardness results, many of the proposed algorithms in the literature work well for practical problems \citep{leurgans1993decomposition, tensor_decomp_latent_anandkumar2014, kolda2009tensor}. In fact, for a wide range of these algorithms there are theoretical local and global guarantees under some realistic assumptions \citep{uschmajew2012local, tensor_decomp_latent_anandkumar2014}. One of the important cases where the problem of finding the decomposition has been very well-studied is the case where the components are orthogonal. \citep{tensor_decomp_latent_anandkumar2014} demonstrates that many problems in practice can be reduced to an orthogonal tensor decomposition with a pre-processing phase known as data whitening. While data whitening helps us transform the problem to the orthogonal case, it is computationally expensive especially in high-dimensional settings. Besides, it can affect the performance of the model for problems such as Independent Component Analysis \citep{le2011ica}.

Practical drawbacks of data whitening, alongside with theoretical concerns such as instability in high-dimensional cases \citep{anandkumar2014guaranteed}, have motivated researchers to investigate the CP tensor decomposition problem in the non-orthogonal scenario. \citep{anandkumar2014guaranteed} provide local and global guarantees for recovering the components of CP under mild non-orthogonality assumptions. However, their result requires on the proper initialization of the algorithm close to the global optimum. \citep{ge2017optimization} analyze the non-convex landscape of the non-orthogonal tensor decomposition problem, and characterize the local minima of the problem under the over-complete regime (rank of the tensor is much higher than the dimension of the components).  \citep{non-orthogona_ALS_sharan_valiant_17} show that the orthogonalized alternating least square approach can \textit{globally} recover the components of the tensor decomposition problem when $k$, the rank of the tensor, is $\mathcal{O} (d ^ {0.25})$ where d is the dimension. \textit{In this paper, we aim to show the global convergence of a variation of Tensor Power Method (TPM) augmented by the deflation and restart of the algorithm for fourth-order tensors.} \textit{This algorithm can recover all components of a given non-orthogonal decomposition problem when} 
$k = \bO(d ^ {0.5})$.

Before proceeding to the main results, let us define some notations. Let $u(i)$ be the $i$-th coordinate of vector $u$, the Kronecker product of $n$ vectors $u_1, u_2, ..., u_n$ denoted by $T = u_1 \otimes u_2 \otimes ... \otimes u_n$, is defined as an $n$th-order tensor $T$, such that $T(i_1, i_2, ..., i_n) = u_1(i_1) u_2(i_2) ... u_n(i_n)$. Moreover,
\begin{gather*}
    u^{\otimes n} \triangleq \underbrace{u \otimes u \otimes ... \otimes u}_{n \: \textrm{times}}
\end{gather*}

A fourth-order tensor $T$, can be seen as a multi-linear transformation, defined for given  d-dimensional vectors $x, y, z,$ and $t$ as
\begin{align}
    T(x, y, z, t) = \sum_{i=1}^{d} \sum_{j=1}^{d} \sum_{k=1}^{d} \sum_{\ell=1}^{d} T(i, j, k, \ell) x(i) y(j) z(k) t(\ell). \label{lambda_map}
\end{align}

To understand the ideas of the proposed algorithm, let us first start by considering the tensor decomposition of a fourth-order tensor in the orthogonal scenario. Suppose that the tensor of interest, $T$, is a fourth-order tensor which can be decomposed as
\begin{align}
    T = \sum_{i=1}^k u_i^{\otimes 4},
\end{align}
where $u_i$'s are $d$-dimensional orthogonal vectors, i.e., $\langle u_i, u_j\rangle=0, ~~i\neq j$. Here, for simplicity of presentation, we assumed that different components have the same weight. For this case, the aim of tensor decomposition is to find the orthogonal decomposition vectors $\{u_i\}_{i=1}^k$ efficiently given the tensor $T$. \citep{tensor_decomp_latent_anandkumar2014} proves that in this case the non-convex optimization problem
\begin{equation}\label{eq: opt1}
\begin{split}
    \min_w \quad &~f(w) = -\frac{1}{4}\sum_{i=1}^k (u_i^T w)^4 = -\frac{1}{4}T(w,w,w,w),\\
    \text{s.t.}\quad &~\|w\|^2 = 1. 
\end{split}
\end{equation}
does not have any spurious local minima. Consequently, all the local minima of this problem correspond to $\pm u_i$, $i=1,\cdots, k$. This property implies that most of the simple first order methods, such as randomly initialized manifold gradient descent, which are proven to converge to local minima \citep{GD_local_min_lee2016} would be able to find the components, almost surely. However, in practice, algorithms such as gradient descent are shown to be slow due to their conservative and static step-size choices. On the other hand, algorithms such as tensor power method (TPM) \citep{tensor_decomp_latent_anandkumar2014} are shown to be practically faster. Moreover, \citep{tensor_decomp_latent_anandkumar2014} shows that TPM with multiple random restarts and deflation is capable of finding  all  components $u_i$ with high probability. Unfortunately, the results in \citep{tensor_decomp_latent_anandkumar2014} assume the orthogonality of the components.

% In this work, we try to extend the results of \citep{tensor_decomp_latent_anandkumar2014} to non-orthogonal components. In a sense, our results show that when the components are not orthogonal, then there is a cap on the number of components for which the nice geometric properties of the optimization problem, i.e. all local minima are global, still holds. We also prove that TPM method with random restarts and deflation is still capable of finding all the components with high probability. 

In this work, we extend the results of \citep{tensor_decomp_latent_anandkumar2014} to the non-orthogonal case. To establish our result, we first analyze the optimization landscape of problem~\eqref{eq: opt1} In Section~\ref{geometric_analysis} under incoherence condition, restricted isometry property (RIP), and a certain upper-bound on the ratio of the weight of different components. We show that any local minimizer of problem~\eqref{eq: opt1} is close to one of the actual components. In other words, for any local minimizer $u^*$ of~\eqref{eq: opt1}, there exists an index $i$, such that $\|u_i - u^*\|$ is small (up to sign ambiguity in $u^*$). 
In Section~\ref{recovery}, we show that Tensor Power Method (TPM) with deflation and restarting can recover all components of a given tensor with high probability. 

%The problem in decomposing a tensor $T = \sum_{i=1}^k u_i^{\otimes 4}$, where $u_i\in \mathbb{R}^d$ and $\|u_i\|=1$ are its components, is to extract $u_i$'s from complete/incomplete/noisy knowledge of $T$. The problem of tensor decomposition is not easy to solve in general. In this work we aim to introduce reasonable scenarios in which such decomposition could be done efficiently. To be more specific, our result suggest that when $k = \bO(\sqrt{d})$, and the components are randomly generated  such a decomposition could be merely done using first order methods such as tensor power method.

\section{Optimality Conditions for the Tensor Decomposition Problem}
% In order to apply first-order methods to decomposing $T = \sum_{i=1}^k u_i^{\otimes 4}$ we use the following optimization problem:

% \begin{align}
%     \min_w~&~f(w) = -\frac{1}{4}\sum_{i=1}^k (u_i^T w)^4 = -\frac{1}{4}T(w,w,w,w),\nonumber\\
%     \text{s.t.}~&~\|w\|^2 = 1. \label{eq: opt1}
% \end{align}
In order to solve problem~\eqref{eq: opt1}, we can use the manifold gradient descent method. It is well-known that manifold gradient descent with random initialization converges to local minima \citep{GD_local_min_lee2016}, almost surely. Thus, if we prove any local minima of the above problem is close to one of the components $u_i$, we can use manifold gradient descent for recovering the components. This gradient descent method for solving tensor decomposition has been used before in the case of orthogonal orthogonal tensors; see, e.g. \citep{tensor_decomp_latent_anandkumar2014}. In fact, in the orthogonal case, the gradient descent method coincides with Tensor Power Method \citep{tensor_decomp_latent_anandkumar2014}.

To study the landscape of~\eqref{eq: opt1}, let us first present the first- and second-order optimality conditions. Let the projection matrix to the manifold $\mathcal{M} = \bigg\{w\bigg|\|w\|=1\bigg\}$ at point $w$  be $P_w = I - ww^T$. Then, for the optimization problem~\eqref{eq: opt1}, any local minimizer point $w$  has to satisfy the following two optimality conditions:
\begin{itemize}
    \item First-order optimality condition\footnote{The subscript $\mathcal{M}$ refers to the fact that the corresponding derivative is calculated while projecting the directions on the manifold $\mathcal{M}$.}:% = \bigg\{w\bigg|\|w\|=1\bigg\}$}:
    \begin{align}
        \nabla_\mathcal{M} f(w) = - P_w\sum_{i=1}^k (w^T u_i)^3 u_i = 0
    \end{align}
    \item Second-order optimality condition:
    \begin{align}
        \nabla^2_\mathcal{M} f(w) = \underbrace{\bigg(\sum_{i=1}^k (u_i^T w)^4 \bigg)}_{-4~f(w)}P_w - 3\sum_{i=1}^k (u_i^Tw)^2 P_w u_i u_i^T P_w \succeq \mathbf{0}
    \end{align}
\end{itemize}
    
The first-order optimality condition implies that for any local optimal point~$w$, we have 
    \begin{align}
        \underbrace{\bigg(\sum_i (w^T u_i)^4\bigg)}_{\lambda} w = \sum_i (w^T u_i)^3 u_i. \label{eq: first_order_result}
    \end{align}
Consequently, $w$ has to be in the span of $u_i$'s if $\lambda \neq 0$.
% \begin{corollary}
%     At a first order stationary solution $w$, for any $i=1,\cdots, k$
%     \begin{align}
%         \lambda (w^T u_i) = (w^T u_i)^3 \|u_i\|^2 + \sum_{j\neq i} (w^T u_j)^3 u_j^T u_i
%     \end{align}
% \end{corollary}
Based on these optimality conditions, we study the landscape of the tensor decomposition problem~\eqref{eq: opt1} in two steps: In the first step, we show that if a local minimum exists close to one of the components, the local minimum should be in fact very close to the true component. In other words, within a region around any true component, the local minima are all very close to the true component. Thus, the landscape is locally well-behaved around the true components. In the second step, we make our result global by showing that any local minimizer of the tensor decomposition problem~\eqref{eq: opt1} is relatively close to one of the true components (up to sign ambiguity).

% study the problem locally. 
% Our analysis comprises of two part, i.e. local and global. In the local analysis we prove that if $w$ is a stationary solution that it relatively close to one component, say $u_1$, then it will be very close to it. Then, in order to make sure that $w$ is close enough for our local analysis to work, we do  global analysis to prove that any local min of \eqref{eq: opt1} is in fact close enough to one of the $\pm u_i$'s.
%\section{Assumptions}
\vspace{0.2cm}

To proceed, let us make the following standard assumptions:

\begin{assumption}\label{ass: incoherence}
    $u_i$'s are all norm $1$ and satisfy the following incoherence condition with constant $\tau$, i.e., 
    \begin{align}
        |u_i^T u_j|< \tau,
    \end{align}
    and $\tau = \bO(\frac{1}{\sqrt{d}})$. 
    Moreover, we assume that for any vector $w$ in the span of $\{u_i\}_{i=1}^k$, we have 
    \begin{align}
        (1-\delta)\|w\|^2 \leq \|U^Tw\|^2\leq (1+\delta) \|w\|^2,
    \end{align}
    where $U = \big[u_1,\cdots,u_k\big]\in \mathbb{R}^{d\times k}$ and $\delta = \bO\left(\sqrt{k/d }\right)$. This condition is known in the literature as Restricted Isometric Property (RIP) that is usually satisfied with high probability when $d$ is large for many forms of random matrices.
\end{assumption}
%\begin{remark}
Furthermore, for general matrices, it is easy to prove that $\delta\leq (k-1)\tau$.
%\end{remark}

\section{Geometric Analysis}\label{geometric_analysis}
Throughout this section, for any given point~$w$, we define $c_i(w) = |w^Tu_i|$. For simplicity of notations and since it is clear from the context, we use $c_i$ instead of $c_i(w)$.  We also, without loss of generality, assume that $|c_1|\geq |c_2|\geq \cdots\geq |c_k|$. Using this definition, the following lemma shows that there is always a gap between $c_1$ and the rest of the components~$c_j$, $j\neq 1$.
\pagebreak
\begin{lemma}
    Let $w$ be a local minimizer of \eqref{eq: opt1}, then $|c_1|> \sqrt{2} |c_2|$. \label{lemma: cgaps}
\end{lemma}
\begin{proof}
    We prove by contradiction. Assume the contrary that $|c_1| \leq \sqrt{2} |c_2|$. Take a unit vector $v$ in the span of $u_1$ and $u_2$ which is orthogonal to $w$. First of all, we have that
    \begin{align}
        v^T\nabla^2_\mathcal{M} f(w) v &= \big(\sum_i (w^Tu_i)^4\big) \underbrace{\|v\|^2}_{=1}- 3 \sum_{i}(u_i^Tw)^2 (u_i^Tv)^2\nonumber\\
        &\leq c_1^2 \underbrace{\big(\sum_i (w^Tu_i)^2\big)}_{\leq (1+\delta)\|w\|^2 = (1+\delta)} - 3 \min_{i=1,2}((u_i^Tw)^2) \|[u_1, u_2]^Tv\|_2^2\nonumber
    \end{align}
    Based on the RIP assumption and since $\|v\|=1$, we have $\|[u_1, u_2]^Tv\|_2^2>(1-\delta)$. Furthermore, $\min_{i=1,2}((u_i^Tw)^2)\geq \frac{1}{2} c_1^2$ based on the contrary assumption we made. Thus, when $\delta\leq k\tau\leq 0.01$ we have
    \begin{align}
        v^T\nabla^2_\mathcal{M} f(w) v \leq c_1^2 (1+\delta) - \frac{3}{2} c_1^2 (1-\delta) \leq c_1^2(1+\delta -\frac{3}{2}(1-\delta))<0. \label{eq:temp1}
    \end{align}
    On the other hand, since $w$ is a local minimizer of~\eqref{eq: opt1}, the second-order optimality condition implies that~$v^T\nabla^2_\mathcal{M} f(w) v\geq 0$, which contradicts~\eqref{eq:temp1}.
\end{proof}

The following lemma shows that any local minimizer $w$ of problem~\eqref{eq: opt1}, is very close one of the to the component $u_1$, with the highest value of $u_i^T w$ for $i \in \{1, \dots, k\}$. In other words, the projection of $w$, onto space spanned by the rest of components is very small compared to its projection onto $u_1$.     
\begin{lemma}
    If $w$ is a local minimizer of~\eqref{eq: opt1} and $k\tau\leq 0.05$ then
    \begin{align}
        \frac{\|\Pperp(w)\|}{|u_1^T w|}\leq \bO(\sqrt{k}\tau^3)
    \end{align} \label{lemma: correlation}
\end{lemma}
\begin{proof}
    %First of all note that based on the above Lemma we know that $|c_1|>\sqrt{2}|c_2|$. Moreover, due to the first order optimality condition \eqref{eq: first_order_result} we know that $w \propto z = \sum_i (w^Tu_i)^3 u_i$. Thus,
    
    Assume that $z = \sum_i (w^Tu_i)^3 u_i$. Based on~\eqref{eq: first_order_result}, we have $z = \lambda w$. Thus:  
    
    \begin{align}
        \frac{\|\Pperp(w)\|}{|w^Tu_1|} = \frac{\|\Pperp(z)\|}{|z^Tu_1|}
    \end{align}
    According to the previous lemma, $|c_1| > \sqrt{2}|c_i|$ for any $i \in \{2, 3, \dots, k\}$. Therefore,
    
    \begin{align}
        |z^Tu_1| = \bigg|c_1^3 + \sum_i c_i^3 u_i^T u_1\bigg| \geq |c_1|^3 - \frac{k\tau}{2\sqrt{2}}|c_1|^3\geq 0.98 |c_1|^3. \label{eq:lemma1_eq1}
    \end{align}
    Moreover,
    \begin{align}
        \|\Pperp(z)\|^2=\bigg\|\Pperp\bigg(\sum_{i\neq 1} (w^Tu_i)^3 u_i\bigg)\bigg\|^2\leq (1+\delta)\sum_{i\neq 1} c_i^6 \nonumber
    \end{align}
    Note that $c_i = c_1 u_1^T u_i + \Pperp(w)^T u_i$. To find an upper-bound for $c_i^6$, we use the following lemma.
    \begin{lemma}
        For any two real numbers $a$ and $b$, we have $(a+b)^6\leq 1.01 a^6 + \bO(b^6)$. \label{lemma: power6}
    \end{lemma}
    \begin{align}
        \|\Pperp(z)\|^2 \leq 1.01 (1 + \delta) \sum_{i\neq 1} (u_i^T \Pperp(w))^6 + \underbrace{\bO(\sum_{i\neq 1}c_i^6\tau^6)}_{\leq c_1^6 \bO(k \tau^6)}
    \end{align}
    Since $|u_i^T \Pperp(w)| \leq |c_i|+|c_1|\tau\leq |c_1|(\frac{1}{\sqrt{2}}+\tau)$ for $i\neq 1$, we have: 
    \begin{align}
        \|\Pperp(z)\|^2 &\leq (1 + \delta)c_1^4 (\frac{1}{\sqrt{2}}+\tau)^4 \underbrace{\sum_{i\neq 1}(u_i^T \Pperp(w))^2}_{\leq (1+\delta)\|\Pperp(w)\|^2} + c_1^6 \bO(k\tau^6)\nonumber\\
        &\leq 2c_1^4\max\bigg((\frac{1}{\sqrt{2}}+\tau)^4(1+\delta)^2\|\Pperp(w)\|^2, c_1^2 \bO(k\tau^6)\bigg) \label{eq:lemma1_eq2}
    \end{align}
    Combining \eqref{eq:lemma1_eq1} and~\eqref{eq:lemma1_eq2}, we obtain: 
    \begin{align}
        \frac{\|\Pperp(w)\|}{\underbrace{|w^Tu_1|}_{|c_1|}} &= \frac{\|\Pperp(z)\|}{|z^Tu_1|} \nonumber\\
        &\leq \frac{1.01}{0.98}\max\bigg(\sqrt{2}\big(\frac{1}{\sqrt{2}}+\tau\big)^2(1+\delta){\frac{\|\Pperp(w)\|}{c_1}}%_{\frac{\|\Pperp(w)\|}{|w^Tu_1|}}
        , \bO(\sqrt{k}\tau^3)\bigg).
    \end{align}
    Now note that
    \begin{align}
        \frac{1.01}{0.98}\sqrt{2}\big(\frac{1}{\sqrt{2}}+\tau\big)^2(1+\delta)\leq 0.9.\nonumber
    \end{align}
    Thus,
    \begin{align}
        \frac{\|\Pperp(w)\|}{|w^Tu_1|} \leq \max \bigg(0.9\frac{\|\Pperp(w)\|}{|w^Tu_1|}, \bO(\sqrt{k}\tau^3)\bigg),
    \end{align}
    which completes the proof.
    
\end{proof}

%\pagebreak

\begin{theorem}
     If $k\tau\leq 0.05$, then for any local minimizer $w$ of \eqref{eq: opt1} there exists an index $i$ such that $\|w-u_i\|\leq \bO(\sqrt{k}\tau^3)$. \label{theorem: geometric_analysis}
\end{theorem}
\begin{proof}
    Based on the above lemma and without loss of generality, assume that $\frac{\|\Pperp(w)\|}{|w^Tu_1|} \leq \bO(\sqrt{k}\tau^3)$. Thus, $\|\Pperp(w)\|^2 = 1-c_1^2 \leq c_1^2 \bO(k\tau^6)$. Therefore,
    \begin{align}
        c_1 \geq \frac{1}{\sqrt{1+\bO(k\tau^6)}} \geq 1-\bO(k\tau^6)\nonumber.
    \end{align}
    Now we have $\|w-u_1\|^2 = 2(1-c_1)\leq \bO(k\tau^6)$.
\end{proof}

%\begin{corollary}
%    A similar result could be easily obtained for the case where original tensor $T = %\sum_{i=1}^k \lambda_i u_i^{\otimes 4}$, for $\lambda_1\geq\cdots\geq\lambda_k>0$.
%\end{corollary}
%\begin{corollary}
%    Note that in the case where the components are randomly generated with dimension $d$ our result suggests that TPM is effective when $k=\bO(\sqrt{d})$.
%\end{corollary}

\begin{corollary}
    Note that in the case where the components are randomly generated with dimension $d$ our result shows that when $k\leq\bO(\sqrt{d})$ there are no spurious local minima.
\end{corollary}

\section{Extension to the Non-equally Weighted Scenario}
In this section, we extend the result of the previous section to the case when the tensor decomposition components have not equal weights. In this scenario, the tensor of interest $T$ is in the form of:

\begin{align}
    T = \sum_{i=1}^k \lambda_i u_i^{\otimes 4}.\nonumber
\end{align}

Thus, the optimization problem~\eqref{eq: opt1} turns to:

\begin{equation}\label{eq: extended_opt1}
\begin{split}
    \min_w \quad &~f(w) = -\frac{1}{4}\sum_{i=1}^k \lambda_i (u_i^T w)^4 = -\frac{1}{4}T(w,w,w,w),\\
    \text{s.t.}\quad &~\|w\|^2 = 1. 
\end{split}
\end{equation}

Let $\lambda_{max}$ and $\lambda_{min}$ be the maximum and minimum values among $\{\lambda_i\}_{i=1}^{k}$, respectively. The following theorem demonstrates that under an additional assumption on the ratio of $\lambda_{max}$ to $\lambda_{min}$, any local minimizer of the optimization problem~\eqref{eq: extended_opt1} is very close to one of the actual components $u_i$.

\begin{theorem}
     If $k\tau\leq 0.05$ and $\kappa = \frac{\lambda_{max}}{\lambda_{min}} \leq \frac{5}{4}$ then for any local minimizer $w$ of \eqref{eq: extended_opt1} there is an index $i$ such that $\|w-u_i\|\leq \bO(\kappa \sqrt{k}\tau^3)$.
\end{theorem}

We leave the optimality conditions of problem~\eqref{eq: extended_opt1}, the extension of lemma~\ref{lemma: cgaps}, lemma~\ref{lemma: correlation}, and the proof of the above theorem to appendix~\ref{Extension}.

%\pagebreak
\section{Recovery via Tensor Power Method and Deflation}\label{recovery}
The landscape analysis in Section~\ref{geometric_analysis} demonstrates that to recover at least one of the true components $u_i$ of a given tensor $T$, with the error of at most $\bO(\kappa \sqrt{k}\tau^3)$, it suffices to find any local optimum of problem~\eqref{eq: extended_opt1}. To recover the other components of $T$, we can deflate the obtained component from $T$, and find a local optimizer of problem~\eqref{eq: extended_opt1} for the deflated tensor. However, the introduced deflation error could potentially make the problem of finding the remaining $u_i$ components difficult. In Section~\ref{sec: TPM} we show that TPM can tolerate error residuals from deflation if it is well-initialized. In other words, TPM can recover all $u_i$'s within a noise floor of $\bO(\kappa \sqrt{k}\tau^3)$. Similar to the geometric result, our convergence guarantee for TPM with "good" initialization is deterministic. 

Finally, in Section~\ref{sec: TPM_init}, we provide a probabilistic argument to show that when $d$ is large after restarting TPM randomly for $L = poly(k, d)$ times, with high probability, we obtain a "good" initialization. This concludes that TPM with restarts and deflation can efficiently find all $u_i$'s within some noise floor $\bO(\sqrt{k}\tau^3)$.

\subsection{Convergence of Well-initialized TPM with Deflation}\label{sec: TPM}
Tensor Power Method is one of the most widely used algorithms for solving tensor decomposition problems. For a given tensor $T$, and vectors $x, y,$ and $z$, consider the following mapping:
\begin{align}
    T(I, x, y, z) = \sum_{i=1}^{d} \sum_{j=1}^{d} \sum_{k=1}^{d} \sum_{\ell=1}^{d} T(i, j, k, \ell) x(j) y(k) z(\ell) e_i, \label{TPM_mapping}
\end{align}
where $e_i$ is the $i$-th standard unit vector. 
 
Assume that a tensor $T$ and initialization $w_0$ is given. At each iteration, TPM applies vector-valued mapping~\eqref{TPM_mapping} to current $w$, which is initialized to $w_0$. Then, it normalizes the resulted vector, and update $\lambda$, by applying mapping~\eqref{lambda_map} to $w$. The details of  Tensor Power Method are summarized in Algorithm \ref{alg: TPM}.

\begin{algorithm}\caption{$TPM(T, w_0, \mathcal{T})$}\label{alg: TPM}
\begin{algorithmic}
\FOR {$t=1,\cdots,\mathcal{T}$}
    \STATE $w_t = T(I, w_{t-1},w_{t-1},w_{t-1})/\|T(I, w_{t-1},w_{t-1},w_{t-1})\|$
    \STATE $\lambda_t = T(w_t, w_{t},w_{t},w_{t})$
\ENDFOR
\RETURN $(w_\mathcal{T}$, $\lambda_\mathcal{T})$
%\STATE $i\gets 10$
%\IF {$i\geq 5$} 
%        \STATE $i\gets i-1$
%\ELSE
%        \IF {$i\leq 3$}
%                \STATE $i\gets i+2$
%        \ENDIF
%\ENDIF 
\end{algorithmic}
\end{algorithm}
The following theorem shows that when $k\tau$ is small enough, applying algorithm~\ref{alg: TPM} to a deflated version of tensor $T$ can recover all $u_i$ components.

%\pagebreak
%Now we are ready to prove our induction based result.
\begin{theorem}\label{thm: TPM}
Fix $1 \leq i\leq k$. Assume that we have access to $\hat{u}_j\in \mathbb{R}^d$ and $\hat{\lambda}_j\in \mathbb{R}$, $j=1,\cdots, i-1$ such that there are appropriate absolute constants $c_1$ and $c_2$ where:
\begin{itemize}
    \item $k\tau\leq c_1$
    \item $\|\hat{u}_j-u_j\|\leq \epsilon = c_2 \sqrt{k}\tau^3$ and $|1-\hat{\lambda}_j|\leq 5 \epsilon$.
\end{itemize}
Moreover, assume that we have an initial unit vector $w_0$, for which $|w_0^Tu_i|\geq \tau$ and $|w_0^Tu_i|\geq 2|w_0^Tu_j|,~\forall~j\neq i$. Define the deflated tensor $T_i = T + \sum_{j=1}^{i-1}-\hat{\lambda}_j \hat{u}_j^{\otimes 4}$. Suppose that we apply TPM for appropriate number of iterations $\mathcal{T}= \bO(1) + \bO(\log(1/\sqrt{k}\tau^4))$ to obtain $(\hat{u}_i, \hat{\lambda}_i) = TPM(T_i, w_0, \mathcal{T})$. Then, we have
\begin{align}
    \|\hat{u}_i-u_i\| \leq \epsilon~~and~~|1-\hat{\lambda}_i|\leq 5\epsilon.
\end{align}

%by running Tensor Power Method starting at $w_0$ after $ \mathcal{T}= \bO(1) + \bO(\log(1/\sqrt{k}\tau^4))$ on Tensor $T_i = T + \sum_{j=1}^{i-1}-\hat{\lambda}_j \hat{u}_j^{\otimes 4}$ we obtain a point $\hat{u}_i = w_{\mathcal{T}}$ such that $\|\hat{u}_i-u_i\| \leq \epsilon$. Moreover, $\hat{\lambda}_i = T_i(\hat{u}_i, \hat{u}_i, \hat{u}_i, \hat{u}_i)$ satisfies $|1-\hat{\lambda}_i|\leq 5\epsilon$.
\end{theorem}
\begin{corollary}
    A similar result can be obtained when the original tensor $T$ is in the form of $T = \sum_{i=1}^k \lambda_i u_i^{\otimes 4}$, for $\lambda_1\geq\cdots\geq\lambda_k>0$.
\end{corollary}

\subsection{Obtaining good initialization by random restarts}\label{sec: TPM_init}
In Section \ref{sec: TPM} we proved that TPM is effective in finding the components of $T$ when applied on deflated tensors sequentially and with good enough initialization. In this section we prove that we can obtain such a good initialization by doing multiple random restart in each iteration. Algorithm \ref{alg: TPMR} describes Tensor Power Method augmented by multiple random \textit{restarts} at each iteration (TPMR). %$L=\bO(\poly(exp(d\tau^2), k))$ random initialization. 
\begin{algorithm}
\caption{$TPMR(T, \mathcal{T}, L, k)$}\label{alg: TPMR}
\begin{algorithmic}
    \FOR{$i=1,\cdots,k$}
        \IF{i=1}
            \STATE $T_i=T$
        \ELSE
            \STATE $T = T_{i-1} - \hat{\lambda}_{i-1} \hat{u}_{i-1}^{\otimes 4}$
        \ENDIF
        \STATE Generate $v_1,\cdots, v_L$ uniformly on $\mathcal{M}$
        \STATE Find $(w_{\ell},\lambda_{\ell}) = TPM(T, v_{\ell}, \mathcal{T})$, $\ell=1,\cdots, L$
        \STATE Find $\ell^* = \arg\max_{\ell}\lambda_{\ell}$
        \STATE Set $\hat{u}_i = w_{\ell^*}$ \& $\hat{\lambda}_i = \lambda_{\ell^*}$
    \ENDFOR
    \RETURN $(\hat{u}_1,\cdots, \hat{u}_k)$
\end{algorithmic}
\end{algorithm}

TPMR calls TPM with multiple restarts on its inside loop. Thus, to prove the effectiveness of TMPR algorithm, it suffices to find good initialization points for TPM using random restarts. The following theorem states such a result.

\begin{theorem}\label{thm: TPM_init}
    For any small threshold $\eta>0$, if we sample $L$ uniform vectors from $\mathcal{M}$ such that $L$ satisfies
    \begin{align}
        &A_1(L, \eta)= 0.5 \sqrt{\log~L} - \sqrt{2\log \frac{12}{\eta}} \geq 2\times B_1(L, \eta, \tau),\label{eq: L_cond_1}\\
        &\frac{A_1(L,\eta)}{C_1(\eta, d)} \geq \tau\label{eq: L_cond_2}
    \end{align}
    where 
    \begin{align}
        &B_1(L, \eta, \tau)= \sqrt{2(1+\tau^2)\log(2k)} %\nonumber\\
        +\tau \Big(\sqrt{2\log(2L)} + \sqrt{2\log\frac{12}{\eta}}\Big)%\nonumber\\
        + \sqrt{2(1+\tau^2)\log\frac{3}{\eta}},\nonumber\\
        &C_1(\eta, d) = \sqrt{3\log(3/\eta)~d + 2\log(3/\eta)}
    \end{align}
    then, with probability $1-\eta$, at least one of the samples $v_\ell\in\{v_1, \cdots, v_L\}$ will satisfy:
    \begin{align}
        &|v_\ell^Tu_1|\geq 2 |v_\ell^T u_i|, \forall i\neq 1\label{eq: init_1}\\
        &|v_\ell^T u_1|\geq \tau\label{eq: init_2}
    \end{align}
\end{theorem}
\begin{corollary}
    To make sure that the good initialization condition is satisfied throughout the for loop in TPMR Algorithm \ref{alg: TPMR} with probability $\eta_0$, we need to plug in $\eta = \eta_0/k$ in Theorem \ref{thm: TPM_init}.
\end{corollary}
\begin{corollary}
    It is easy to verify that one can find $L = Poly(exp(d\tau^2), k)$ that satisfies \eqref{eq: L_cond_1} and \eqref{eq: L_cond_2}.
\end{corollary}
\begin{remark}
    Note that in the case where $u_i$'s are randomly generated $d$ dimensional vectors, then $\tau\leq \bO(1/\sqrt{d})$. Thus, $\exp(d\tau^2) = \bO(d)$. Thus, $L = Poly(d, k)$.
\end{remark}

\pagebreak
\bibliographystyle{abbrvnat}
\bibliography{references}

\pagebreak
\appendix
\section{Helper Lemmas for Proving TPM Convergence}
%First of all let us prove some helper lemmas:
% 3rd order version
%\begin{lemma}
%    If $\|u_i-\hat{u}_i\|\leq \epsilon$ and $|1-\hat{\lambda}_i|\leq \epsilon$, for any $w$ such that $\|w\|=1$, then
%    \begin{align}
%        E_i(I, w, w) = (w^Tu_i)^2 u_i- (w^T\hat{u}_i)^2 \hat{u}_i = A_i u_i - B_i \hat{u}_i^{\perp}/\|\hat{u}_i^{\perp}\|,
%    \end{align}
%    where $\hat{u}_i^{\perp} = \hat{u}_i- (u_i^T\hat{u}_i)u_i$ and 
%    \begin{align}
%        |A_i|\leq 7 |c_i|\epsilon + 2\epsilon^2~~\&~~ |B_i|\leq 4 c_i^2\epsilon + 4\epsilon^3,
%    \end{align}
%    where $c_i = w^Tu_i$.
%\end{lemma}

\begin{lemma}\label{lemma: helper_1}
    If $\|u_i-\hat{u}_i\|\leq \epsilon$ and $|1-\hat{\lambda}_i|\leq 5\epsilon\leq 1$, then for any $w$ such that $\|w\|=1$, 
    \begin{align}
        E_i(I, w, w, w) = (w^Tu_i)^3 u_i- \hat{\lambda}_i(w^T\hat{u}_i)^3 \hat{u}_i = A_i u_i - B_i \hat{u}_i^{\perp}/\|\hat{u}_i^{\perp}\|,
    \end{align}
    where $\hat{u}_i^{\perp} = \hat{u}_i- (u_i^T\hat{u}_i)u_i$ and 
    \begin{align}
        |A_i|\leq 7 |c_i|^3\epsilon + 10|c_i|\epsilon^2 + 2\epsilon^2~~\&~~ |B_i|\leq 8 |c_i|^3\epsilon + 8\epsilon^4,
    \end{align}
    where $c_i = w^Tu_i$.
\end{lemma}
\begin{proof}
    Let us use the following definitions: $c_i = w^Tu_i$, $a_i = u_i^T \hat{u}_i$ and $b_i = w^T(\huip/\|\huip\|)$. Then, we have
    \begin{align}
        E_i(I, w, w, w) &= c_i^3 u_i- \hat{\lambda}_i(a_i c_i + \|\huip\| b_i)^3(a_i u_i + \huip)\nonumber\\
        &= \underbrace{c_i^3 - \hat{\lambda}_i  a_i (a_i c_i + \|\huip\| b_i)^3}_{A_i} u_i - \underbrace{\hat{\lambda}_i \|\huip\|(a_i c_i + \|\huip\| b_i)^3}_{B_i} \frac{\huip}{\|\huip\|}
    \end{align}
    Also, note that $1-a_i = \frac{\|u_i-\hat{u}_i\|^2}{2}\leq \epsilon^2/2$. In addition, $\|\huip\|^2 \leq \|u_i-\hat{u}_i\|^2\leq \epsilon^2$ and $|a_i|, |b_i|\leq 1$. Note that we can easily prove that $1-a_i^4 \leq 2\epsilon^2$.
    Now for $A_i$ we have
    \begin{align}
        |A_i| &= |c_i^3 (1-\hat{\lambda}a_i^4) - \hat{\lambda}_i a_i (\|\huip\|^3 b_i^3 + 3 a_i c_i \|\huip\|^2 b_i^2 + 3 a_i^2 c_i^2 \|\huip\|b_i)|\nonumber\\
        &\leq |c_i|^3 (|1-\hat{\lambda}_i| + |\hat{\lambda}_i||1-a_i^4|) + \underbrace{(1+5\epsilon)}_{\leq 2}(\epsilon^3 + 3 \epsilon^2 |c_i| + 3 \epsilon c_i^2)\nonumber\\
        &\leq |c_i|^3(5\epsilon + 4\epsilon^2) + 6 c_i^2\epsilon + 6 |c_i| \epsilon^2 + 2 \epsilon^3\nonumber\\
        &\leq 11 |c_i|^2\epsilon + 10|c_i|\epsilon^2 + 2\epsilon^3,
    \end{align}
    where the last step follows from $|c_i|\leq 1$.
    
    For bounding $B_i$ we use the fact that for $\alpha,\beta\geq 0$ $(\alpha+\beta)^3 \leq 4 \alpha^3 + 4\beta^3$.
    \begin{align}
        |B_i| \leq (1+5\epsilon)\epsilon (|c_i| + \epsilon)^3 \leq 8\epsilon (|c_i|^3+\epsilon^3)
    \end{align}
\end{proof}

\begin{lemma}\label{lemma: helper_0}
    If $\|u_i-\hat{u}_i\|\leq \epsilon$ and $|1-\hat{\lambda}_i|\leq 5\epsilon\leq 1$ for any $i$, then for any $w$ such that $\|w\|=1$, we have: 
    \begin{align}
        \Big\|\bar{E}_i(I, w, w, w)\Big\|^2=\Big\|\sum_{j=1}^i E_j(I, w, w, w)\Big\|^2 \leq 2(1+\delta)\sum_{j=1}^i A_j^2 + 2\Big(\sum_{j=1}^i |B_j|\Big)^2.
    \end{align}
    Thus, we can conclude that
    \begin{align}
        \Big\|\bar{E}_i(I, w, w, w)\Big\|\leq \Big(67 \epsilon\sqrt{\sum_{j=1}^i c_j^4} + 40 \epsilon^2\sqrt{\sum_{j=1}^i c_j^2}  + 16 \epsilon \sum_{j=1}^i |c_j|^3+  8 \sqrt{k}\epsilon^3 + 16 {k}\epsilon^4 \Big), \label{eq: error_bound_1}
    \end{align}
    where if we further have $\epsilon$ small enough, i.e. $\epsilon\leq 350\sqrt{k}\tau^3$ and $\delta \leq k\tau \leq 0.01$, we will have
    \begin{align}
        \Big\|\bar{E}_i(I, w, w, w)\Big\|\leq 100 \epsilon.
    \end{align}
\end{lemma}
\begin{proof}
    The proof of the first part is very simple and only uses the the results of Lemma \ref{lemma: helper_1} and the fact that $u_i$'s satisfy the RIP condition. 
    
    The rest of the proof is also simple arithmetic.
\end{proof}

\begin{lemma}\label{lemma: helper_2}
    If $w$ is a unit vector, $|c_i|\geq |c_j|$, $\forall j$, and moreover $|c_i|\geq \tau$, and $\epsilon \leq 350 \sqrt{k}\tau^3$, then if $k\tau \leq 0.01$ 
    \begin{align}
        \Big\|\bar{E}_{i-1}(I, w, w, w)\Big\|\leq 0.09 |c_i|^3 \label{eq: bound_error_const}
    \end{align}
\end{lemma}
\begin{proof}
We assume that $\tau\leq k\tau \leq 0.01$. Let us look at each term in the right hand side of \eqref{eq: error_bound_1}. For the first term
\begin{align}
    67\epsilon\sqrt{\sum_{j=1}^{i-1} c_j^4} \leq 67 |c_i|^3 \times {k}\tau^2 \leq 2.35 \times 10^{-2} |c_i|^3
\end{align}
For the second term 
\begin{align}
    40 \epsilon^2\sqrt{\sum_{j=1}^i c_j^2} \leq 4.9 \times 10^6 k\sqrt{k}\tau^4 |c_1|^3 \leq 4.9\times 10^{-2}  |c_1|^3.
\end{align}
For the third term we have
\begin{align}
    16 \epsilon \sum_{j=1}^i |c_j|^3 \leq 5.6\times 10^{-3} |c_1|^3
\end{align}
For the fourth term
\begin{align}
     8 \sqrt{k}\epsilon^3 \leq 9.8\times 10^{-7} |c_1|^3
\end{align}
And for the last term we have
\begin{align}
    16 {k}\epsilon^4 \leq 2.5\times 10^{-7} |c_1|^3
\end{align}

Finally the bound could be obtained by adding all these bounds.% and noting $\alpha\leq 200$.
%\begin{align}
%    \epsilon \sum_{j=1}^{i-1} a_j^2 \leq k a_i^2\bO(\sqrt{k}\tau^2) \leq a_i^2 \bO(k^2\tau^2)\leq 0.001 a_i^2
%\end{align}
%For the third term
%\begin{align}
%    \sqrt{k}\epsilon^2 \leq k\sqrt{k}\tau^4 \leq a_i^2 k^2\tau^2 \leq 0.001 a_i^2
%\end{align}
%For the last term
%\begin{align}
%    k \epsilon^3 \leq k^{2.5} \tau^6 \leq a_i^2 k^4\tau^4 \leq 0.001 a_i^2.
%\end{align}

\end{proof}

Now let us prove a recursive bound that we can use for the proving our final result.
\begin{lemma}\label{lemma: recursion}
    Assume that for a norm 1 vector $w$, $|c_i| \geq  2|c_j| $, $\forall j\neq i$ and moreover $|c_i|\geq \tau$. Also assume that the conditions of Lemma \ref{lemma: helper_2} is satisfied and $\epsilon\leq 350 \sqrt{k}\tau^3$. Then for $w_+ = \sum_{j=i}^k c_j^3 u_j + \bar{E}_{i-1}(I, w, w, w)$ we have
    \begin{align}
        \frac{\|P_{u_i^\perp}(w_+)\|}{|w_+^Tu_i|}\leq 0.95 \frac{\|P_{u_i^\perp}(w)\|}{|w^Tu_i|} + 3\|\bar{E}_{i-1}(I, w, w, w)\|+ 15\sqrt{k}\tau^3 \leq 0.95 \frac{\|P_{u_i^\perp}(w)\|}{|w^Tu_i|} + \bO(\sqrt{k}\tau^3) .\label{eq: recurssion_result}
    \end{align}
    Moreover, $|w_+^Tu_i|\geq  2|w_+^Tu_j|$, $\forall j\neq i$.
\end{lemma}

\begin{proof}
    Let us first lower bound $|w_+^Tu_i|$ using Lemma \ref{lemma: helper_2} and the fact that $k\tau \leq 0.01$ we have:
    \begin{align}
        |u_+^Tu_i| \geq |c_i|^3 - \sum_{j=i+1}^k |c_j|^3 \tau - 0.09 |c_i|^3 \geq 0.9 |c_i|^3. \label{eq: lower_bound_a_i}
    \end{align}
    Let us define $z = \sum_{j=i}^k c_j^3 u_j$. Then, $\|P_{u_i^\perp}(w_+)\| \leq \|P_{u_i^\perp}(z)\| + \|\bar{E}_{i-1}(I, w, w, w)\|$.
    \begin{align}
        \|P_{u_i^\perp}(z)\|^2=\bigg\|P_{u_i^\perp}\bigg(\sum_{j= i+1}^k c_j^3 u_i\bigg)\bigg\|^2\leq (1+\delta)\sum_{j=i+1}^k c_j^6 
    \end{align}
    Note that $c_j = c_i u_i^T u_j + P_{u_i^\perp}(w)^T u_j$. And $c_j^6 \leq 200 c_j^6 \tau^6 + 9 (P_{u_i^\perp}(w)^T u_j)^6$. Also note that $|u_j^T P_{u_i^\perp}(w)| \leq |c_j|+|c_i|\tau\leq |c_i|(0.5+\tau)$ for $j>i$. Thus,
    \begin{align}
        \|P_{u_i^\perp}(z)\|^2 &\leq (1+\delta)c_i^4\Bigg( 9 (0.5+\tau)^4 \underbrace{\sum_{i\neq 1}(v_j^T P_{v_i^\perp}(u))^2}_{\leq (1+\delta)\|P_{v_i^\perp}(u)\|^2} + 200 c_i^2 \bO(k\tau^6)\Bigg)\nonumber\\
        &\leq c_i^4 (202 k\tau^6 c_i^2 + 0.63 \|P_{u_i^\perp}(w)\|^2)%a_i^2\bigg((0.5+\tau)^2(1+\delta)\|P_{v_i^\perp}(u)\|^2 + a_i^2 \bO(k\tau^4)\bigg)
    \end{align}
    Therefore,
    \begin{align}
        \frac{\|P_{u_i^\perp}(z)\|}{|u_i^T w_+|}\leq 0.8 \frac{\|P_{u_i^\perp}(w)\|}{|c_i|} + 15\sqrt{k}\tau^3 %\frac{(0.5 + \tau)\sqrt{1+\delta}}{0.98}\frac{\|P_{v_i^\perp}(u)\|}{|a_i|}+ \bO(\sqrt{k}\tau^2)\leq 0.6\frac{\|P_{v_i^\perp}(u)\|}{|a_i|}+ \bO(\sqrt{k}\tau^2)
        \label{eq: z_bound}
    \end{align}
    To bound the other part, i.e. $\frac{\|\bar{E}_{i-1}(I,w,w,w)\|}{|u_i^T w_+|}$, we use the following lemma.
    \begin{lemma}
        Under the assumptions of Lemma \ref{lemma: recursion}, we have
        \begin{align}
            \frac{\|\bar{E}_{i-1}(I,w,w,w)\|}{|u_i^T w_+|} \leq  0.15\frac{\|P_{u_i^\perp}(w)\|}{|c_i|} + 3 \|\bar{E}_{i-1}(I,w,w,w)\| \leq   0.15\frac{\|P_{u_i^\perp}(w)\|}{|c_i|} + \bO(\sqrt{k}\tau^3).
        \end{align}
    \end{lemma}
    \begin{proof}
    Let us consider two cases:
    \begin{itemize}
        \item If $|c_i|\leq 0.8$, then $\|P_{u_i^\perp}(w)\| = \sqrt{1-c_i^2}\geq 0.6$, thus
        \begin{align}
            \frac{\|\bar{E}_{i-1}(I,w,w,w)\|}{|u_i^T w_+|} &\leq \frac{\|\bar{E}_{i-1}(I,w,w,w)\|}{0.9|c_i|^3}\nonumber\\
            &\leq \frac{\|P_{u_i^\perp}(w)\|}{0.9|c_i|} \frac{\|\bar{E}_{i-1}(I,w,w,w)\|}{|c_i| \|P_{u_i^\perp}(u)\|} \leq \frac{1}{0.9\times 0.6} \frac{\|P_{u_i^\perp}(w)\|}{|c_i|} \frac{\|\bar{E}_{i-1}(I,w,w,w)\|}{c_i^2}\nonumber\\
            &\leq \frac{0.09\times 0.8}{0.9\times0.6}\frac{\|P_{u_i^\perp}(w)\|}{|c_i|}\leq 0.15 \frac{\|P_{u_i^\perp}(w)\|}{|c_i|}~~\text{(using Lemma \ref{lemma: helper_2})}%\frac{\|\bar{E}_{i-1}(I,u,u)\|}{\tau}
        \end{align}
%        Note that based on \eqref{eq: error_bound_1}, $\|\bar{E}_{i-1}(I,u,u)\| \leq \bO(\epsilon) \leq \bO(\sqrt{k} \tau^2)$. Thus,
%        \begin{align}
%            \frac{\|\bar{E}_{i-1}(I,u,u)\|}{|v_i^T u_+|} \leq 0.02 \frac{\|P_{v_i^\perp}(u)\|}{|a_i|}
%        \end{align}
        \item $|a_i|\geq 0.8$, then 
        \begin{align}
            \frac{\|\bar{E}_{i-1}(I,w,w,w)\|}{|u_i^T w_+|} &\leq \frac{\|\bar{E}_{i-1}(I,w,w,w)\|}{0.9|c_i|^3} \leq 3 \|\bar{E}_{i-1}(I,w,w,w)\| \leq \bO(\sqrt{k} \tau^2),
        \end{align}
        where the last inequality is due to \eqref{eq: error_bound_1}.
    \end{itemize}
    Combining these two cases, the result is obvious.
    \end{proof}
    Now we can combine the result of this lemma with our bound in \eqref{eq: z_bound}:
    \begin{align}
        \frac{\|P_{u_i^\perp}(w_+)\|}{|w_+^Tu_i|}\leq 0.95 \frac{\|P_{u_i^\perp}(w)\|}{|w^Tu_i|} + 15 \sqrt{k}\tau^3 + 3 \|\bar{E}_{i-1}(I,w,w,w)\| \leq 0.95 \frac{\|P_{u_i^\perp}(w)\|}{|w^Tu_i|} + \bO(k\tau^3).
    \end{align}
    Now we need to make sure the initial conditions of the recursion also hold for normalized $\bar{w}_+ = w_+/\|w_+\|$. First of all as absolute value of the angle between the $w_+$ and $u_i$ is decreasing it is obvious that $|\bar{w}_+^Tu_i|\geq |w^Tu_i|\geq \tau$. So, we only need to prove the gap condition
    \begin{align}
        \frac{|w_+^Tu_i|}{|w_+^Tu_j|}\geq 2, ~~\forall j\neq i.
    \end{align}
    From \eqref{eq: lower_bound_a_i} we know $|w_+^Tu_i|\geq 0.9 |c_i|^3$. Note that $|w_+^Tu_j| \leq |z^T u_j| + \|\bar{E}_{i-1}(I,w,w,w)\|$. Let us consider two cases from here:
    \begin{itemize}
        \item $j>i$: Then, $|z^Tu_j|\leq \sum_{\ell=i}^k |c_\ell|^3 |u_j^Tu_\ell| \leq (\frac{1}{8}+k\tau) |c_i|^3 \leq \frac{1}{5}|c_i|^3$. Now we can use \eqref{eq: bound_error_const}, i.e. $\|\bar{E}_{i-1}(I,w,w,w)\|\leq 0.1|c_i|^3$. Combining these facts it is obvious that $\frac{|w_+^Tu_i|}{|w_+^Tu_j|} \geq \frac{0.9}{\frac{1}{5} + 0.1}\geq 2$.
        \item $j<i$: In this case, $|z^Tu_j|\leq \sum_{\ell=i}^k |c_\ell|^3 |u_j^Tu_\ell| \leq (k\tau) |c_i|^3 \leq 0.01 |c_i|^3$. Now, if we use \eqref{eq: bound_error_const} we get $\frac{|w_+^Tu_i|}{|w_+^Tu_j|} \geq \frac{0.9}{0.1+0.01}\geq 2$.
    \end{itemize}
    
\end{proof}
Thus, we have proved that the if the initial conditions hold, then $\bar{w}_+$ gets closer to $u_i$ , unless it is already very $\epsilon=\bO(\sqrt{k}\tau^3)$ close to $u_i$, see \eqref{eq: recurssion_result}, and also $\bar{w}_+$ satisfies the initial conditions. Before proving our final convergence result, we need the following small lemma which is also a simple lemma to control the final error.
\begin{lemma}\label{lemma: helper_3}
Assume that $\|u_j-\hat{u}_j\|\leq \epsilon$ and $|1-\hat{\lambda}_j|\leq 5\epsilon$, where $ \epsilon\leq 350 \sqrt{k} \tau^3$ for all $j<i$. If for some norm 1 vector $w$, $|c_j|\leq 4\tau$ for all $j\neq i$ and $k\tau \leq 0.01$, then we have 
\begin{align}
    \|\bar{E}_i(I, w, w, w)\| \leq 0.25 \sqrt{k}\tau^3
\end{align}
\end{lemma}
\begin{proof}
The proof can be easily obtained by plugging in the assumptions in \eqref{eq: error_bound_1}. 
%For the first term
%\begin{align}
%    67\epsilon\sqrt{\sum_{j=1}^{i-1} c_j^4} \leq 67 |c_i|^3 \times {k}\tau^2 \leq 6.7 \times 10^{-3} |c_i|^3
%\end{align}
%For the second term 
%\begin{align}
%    40 \epsilon^2\sqrt{\sum_{j=1}^i c_j^2} \leq 40 10^4 k\sqrt{k}\tau^4 |c_1|^3 \leq 4\times 10^{-3}  |c_1|^3.
%\end{align}
%For the third term we have
%\begin{align}
%    16 \epsilon \sum_{j=1}^i |c_j|^3 \leq 16\times 10^{-4} |c_1|^3
%\end{align}
%For the fourth term
%\begin{align}
%     8 \sqrt{k}\epsilon^2 \leq 8\times 10^{-8} |c_1|^3
%\end{align}
%And for the last term we have
%\begin{align}
%    16 {k}\epsilon^4 \leq 16\times 10^{-5} |c_1|^3
%\end{align}

%Finally the bound could be obtained by adding all these bounds.
\end{proof}

\section{Proof of Theorem~\ref{thm: TPM}}

\begin{proof}
Throughout the proof we assume $c_1 = 10^{-3}$ and $c_2 = 350$. We also note that $\delta \leq k\tau$. As $\epsilon \leq 350 \sqrt{k}\tau^3$, we know that $\|\bar{E}_i(I, w, w, w)\| \leq 35000 \sqrt{k}\tau^3$ based on Lemma \ref{lemma: helper_0}. Now define the ratio $r_t = \frac{\|P_{u_i}(w_t)\|}{|w_t^Tu_i|}$.
First of all based on our assumptions on the initialization, $r_0 \leq 1/\tau$. Moreover, based on the recursion that we proved in Lemma \ref{lemma: recursion}, we know 
\begin{align}
   r_{t+1}\leq 0.95 r_t + 3\|\bar{E}_{i-1}(I, w, w, w)\|+ 15 \sqrt{k}\tau^2 \leq 0.95 r_t + 1.06 \times 10^5\sqrt{k}\tau^3
\end{align}
Now if we open up this recursion and use $e$ to denote $e = 1.06 \times 10^5\sqrt{k}\tau^3$ we have
\begin{align}
    r_{t} \leq \frac{0.95^t}{\tau} + 20 e.
\end{align}
So, if we run tensor power method with initialization $w_0$ for $t_0 = \lceil\frac{\log(0.5e\tau)}{\log(0.95)}\rceil$, we would have $r_{t_0} \leq 20.5 e$ and thus, $\|w_{t_0}-u_i\|\leq 21e \leq 3\tau$ as $k\tau\leq 0.001$. As a result, for any $j\neq i$, $|w_{t_0}^Tu_j| \leq 4\tau$. It is obvious that these conditions would hold for any $t\geq t_0$. Therefore, using Lemma \ref{lemma: helper_3} we have
\begin{align}
    \|\bar{E}_i(I, w_t, w_t, w_t)\| \leq 0.25 \sqrt{k}\tau^3, ~~\forall~t\geq t_0.
\end{align}
In the light of this new bound for the error, we can re-write the unrolled version of the recursion in \eqref{eq: recurssion_result} for $t\geq t_0$ as
\begin{align}
    r_t \leq 21e \times (0.95)^{t-T} + 20{\hat{e}},~\forall t \geq t_0
\end{align}
where $\hat{e} = 15.75 \sqrt{k}\tau^3$. Now after ${\mathcal{T}} = t_0 + \lceil \frac{\log(\frac{0.5 \hat{e}}{21e})}{\log(0.95)}\rceil$, we would have $r_{{\mathcal{T}}} \leq 20.5 \hat{e} \leq 325 \sqrt{k}\tau^3$. Thus, $\|w_{{\mathcal{T}}}-u_i\|\leq \epsilon = 350\sqrt{k}\tau^3$. Now by plugging in the $w_{{\mathcal{T}}}$ in tensor $T_i = T + \bar{E}_i$ we get
\begin{align}
    \hat{\lambda}_i = T_i(w_{{\mathcal{T}}},w_{{\mathcal{T}}},w_{{\mathcal{T}}},w_{{\mathcal{T}}}).
\end{align}
It would be easy to check that $|1-\hat{\lambda}_i|\leq 5\epsilon$. Finally, note that $\hat{e}/e$ is constant and thus the number of required iterations ${\mathcal{T}}$ is of the order $\bO(1) + \bO(\log(\frac{1}{\sqrt{k}\tau^4}))$.

\end{proof} 

\section{Proof of Theorem~\ref{thm: TPM_init}}

\begin{proof}
First define $Z_{\ell,i} = v_\ell^T u_i$. It is obvious that $Z_{\ell,i}$'s are Gaussian. Let us define the following probability events. 
\begin{align}
    \Xi_1 = \Bigg\{ Z:&~\max_\ell~|Z_{\ell,1}|\geq A_1(L, \eta) ~~~\&~~~ \nonumber\\
    &~\max_\ell~|Z_{\ell,1}|\leq \sqrt{2\log~2L}  + \sqrt{2\log \frac{12}{\eta}}\Bigg\}\\
    \Xi_{2,\ell} = \Bigg\{Z_\ell:&~\max_{i\neq 1}~|Z_{\ell,i}| \leq B_1(L, \eta, \tau)%\sqrt{2(1+\tau^2)\log(2k)} +\nonumber\\
    %&~~~~~~~~~~~~~~~~~~~~~~~~~~\tau \Big(\sqrt{2\log(2L)} + \sqrt{2\log\frac{12}{\eta}}\Big) + %\sqrt{2(1+\tau^2)\log\frac{3}{\eta}} 
    \Bigg\}\\
    \Xi_{3,\ell} = \Bigg\{v_\ell:&~\|P_{\{u_1,\cdots,u_k\}^\perp}(v_\ell)\|^2 \leq C_1(\eta, d)^2%3\log(3/\eta)~d + 2\log(3/\eta)
    \Bigg\}
\end{align}
It would be easy to see that if $\ell^* = \arg\max_\ell~|Z_{\ell,1}|$ and $\Xi_1\cap \Xi_{2,\ell^*} \cap \Xi_{3,\ell^*}$ happens, then the final result would be true when $L=\bO(k^2)$. So, we just need to prove that $P(\Xi_1\cap \Xi_{2,\ell^*} \cap \Xi_{3,\ell^*})\geq 1-\eta$. To do so, note that $P(\Xi_1\cap \Xi_{2,\ell^*} \cap \Xi_{3,\ell^*}) =P(\Xi_1)P(\Xi_{2,\ell^*} | \Xi_1) P(\Xi_{3,\ell^*} |\Xi_{2,\ell^*}\cap \Xi_1 ) = P(\Xi_1)P(\Xi_{2,\ell^*} | \Xi_1) P(\Xi_{3,\ell^*})$, where the last inequality is due to the independence of $P_{\{u_1,\cdots,u_k\}^\perp}(v_\ell)$ with respect to $Z$'s.
Note that for any $\ell$, it can be easily proved that $P(\Xi_{3,\ell^*}) \geq 1-\eta/3$.

As the initializations are independent, it is well known that
\begin{align}
    P\bigg(|Z_{\ell,1}|\leq 0.5 \sqrt{\log~L}-t\bigg)\leq 2\exp{-t^2/2}\nonumber\\
    P\bigg(|Z_{\ell,1}|\geq \sqrt{2\log(2L)}+ t\bigg)\leq 2\exp{-t^2/2}\nonumber
\end{align}
Using these concentrations, it is clear that with choice of $t=\sqrt{2\log(12/\eta)}$, we get $P(\Xi_1) \geq 1-\eta/3$. 

The only remaining part to prove is to show $P(\Xi_{2,\ell^*} | \Xi_1) \geq 1-\eta/3$. To do that, note that given $Z_{e\ll^*,1}$, $Z_{\ell^*,-1} = [Z_{\ell^*,2},\cdots, Z_{\ell^*,k}]^T$ could be written as Gaussian random variable
\begin{align}
    Z_{\ell^*,-1}| Z_{\ell^*,1}~\sim~\mathcal{N}(\rho  Z_{\ell^*,1}, C-\rho\rho^T),
\end{align}
where $C = \mathbb{E}(Z_{\ell^*,-1}Z_{\ell^*,-1}^T)$ and $\rho$ is the correlation vector, i.e. $\rho = \mathbb{E}(Z_{\ell^*,-1}Z_{\ell^*,1})$. With some abuse of notation, we can write that $Z_{\ell^*,i}| Z_{\ell^*,1} = \rho_i Z_{\ell^*,1} + e_i$, where $e_i$ is Gaussian, zero mean with variance $C_{ii} - \rho_i^2\leq \sigma^2 = 1+\tau^2\leq 1.1$.

Thus, we can write 
\begin{align}
    P\Bigg(&\|Z_{\ell^*,-1}\|_\infty \geq \sqrt{2\sigma^2\log(2k)} + \tau \Big(\sqrt{2\log(2L)} + \sqrt{2\log\frac{12}{\eta}}\Big) + \sqrt{2\sigma^2\log\frac{3}{\eta}} ~\Bigg|~Z_{\ell^*,1}\Bigg)\leq\nonumber\\
    & P\Bigg(\|e\|_\infty \geq \sqrt{2\sigma^2\log(2k)} + \sqrt{2\sigma^2\log\frac{3}{\eta}} \Bigg) \leq \sum_i P\Bigg(|e_i|\geq \sqrt{2\sigma^2\log(2k)} + \sqrt{2\sigma^2\log\frac{3}{\eta}}\Bigg)\nonumber\\
    &\leq \eta/3.
\end{align}
This completes the proof.
\end{proof}

\section{Extending the Geometrical Results to the Non-equally Weighted Scenario }\label{Extension}

In this appendix, we extend the geometrical results demonstrated in Section~\ref{geometric_analysis} to the case where the weights of components are not equal to one necessarily. Suppose that a given tensor $T$ can be decomposed as $T = \sum_{i=1}^k \lambda_i u_i^{\otimes 4}$ where $u_i$'s are unit and $\lambda_i$'s are non-negative. Optimization problem~\eqref{eq: opt1} can be generalized to:

\begin{align}
    \min_w~&~f(w) = -\frac{1}{4}\sum_{i=1}^k \lambda_i (u_i^T w)^4 = -\frac{1}{4}T(w,w,w,w),\nonumber\\
    \text{s.t.}~&~\|w\|^2 = 1. \label{eq: ext_opt2}
\end{align}

The first, and second order conditions for this problem are as follows:
\begin{itemize}
    \item First-order optimality condition:
    \begin{align}
        \nabla_\mathcal{M} f(w) = - P_w\sum_{i=1}^k \lambda_i (w^T u_i)^3 u_i = 0
    \end{align}
    \item Second-order optimality condition:
    \begin{align}
        \nabla^2_\mathcal{M} f(w) = \underbrace{\bigg(\sum_{i=1}^k \lambda_i (u_i^T w)^4 \bigg)}_{-4~f(w)}P_w - 3\sum_{i=1}^k \lambda_i (u_i^Tw)^2 P_w u_i u_i^T P_w \succeq \mathbf{0}
    \end{align}
\end{itemize}

As a consequence of the first-order optimality condition, for any stationary point of problem\eqref{eq: ext_opt2} we have:    
\begin{align}
    \underbrace{\bigg(\sum_i \lambda_i (w^T u_i)^4\bigg)}_{\lambda} w = \sum_i \lambda_i (w^T u_i)^3 u_i. \label{eq: first_order_result2}
\end{align}
    
\begin{lemma}
    Let $w$ be a local minimizer of \eqref{eq: ext_opt2}, then $|c_1|> \sqrt{2} |c_2|$.
\end{lemma}
\begin{proof}
    To arrive in a contradiction assume that $|c_1| \leq \sqrt{2} |c_2|$/
    Take a unit vector $v$ in the span of $u_1$ and $u_2$ which is orthogonal to $w$. 
    \begin{align}
        v^T\nabla^2_\mathcal{M} f(w) v &= \big(\sum_i \lambda_i (w^Tu_i)^4\big) \underbrace{\|v\|^2}_{=1}- 3 \sum_{i} \lambda_i (u_i^Tw)^2 (u_i^Tv)^2\nonumber\\
        &\leq c_1^2 \lambda_{max} \underbrace{\big(\sum_i (w^Tu_i)^2\big)}_{\leq (1+\delta)\|w\|^2 = (1+\delta)} - 3 \min_{i=1,2}\lambda_i ((u_i^Tw)^2) \|[u_1, u_2]^Tv\|_2^2\nonumber\\
    \end{align}
    Based on our assumptions $\|[u_1, u_2]^Tv\|_2^2>(1-\delta)$ as $\|v\| = 1$. Also note that $\min_{i=1,2}((u_i^Tw)^2)\geq \frac{1}{2} c_1^2$ based on our assumption. Thus, when $\delta\leq k\tau\leq 0.01$ we have
    \begin{align}
        v^T\nabla^2_\mathcal{M} f(w) v \leq c_1^2 \lambda_{max} (1+\delta) - \frac{3}{2} c_1^2 \lambda_{min} (1-\delta) \leq \\ c_1^2((\lambda_{max} - \frac{3}{2} \lambda_{min}) + \delta (\lambda_{max} + \frac{3}{2} \lambda_{min}))<0,
    \end{align}
    which is a contradiction with the second-order optimality condition for $w$ that states $v^T\nabla^2_\mathcal{M} f(w) v\geq 0$. Note that the last inequality holds when $\kappa = \frac{\lambda_{max}}{\lambda_{min}} < \frac{3}{2} $ and $\delta < \frac{ \frac{3}{2} \lambda_{min}- \lambda_{max} }{\lambda_{max} + \frac{3}{2} \lambda_{min}}$.
\end{proof}

\begin{lemma}
    If $w$ is a local minimizer of \eqref{eq: ext_opt2}, $k\tau\leq 0.05$, and $\kappa = \frac{\lambda_{max}}{\lambda_{min}} \leq \frac{5}{4}$, then
    \begin{align}
        \frac{\|\Pperp(w)\|}{|u_1^T w|}\leq \bO(\kappa \sqrt{k}\tau^3)
    \end{align}
\end{lemma}
\begin{proof}
    First of all note that based on the above Lemma we know that $|c_1|>\sqrt{2}|c_2|$. Moreover, due to the first order optimality condition \eqref{eq: first_order_result} we know that $w \propto z = \sum_i \lambda_i (w^Tu_i)^3 u_i$. Thus,
    \begin{align}
        \frac{\|\Pperp(w)\|}{|w^Tu_1|} = \frac{\|\Pperp(z)\|}{|z^Tu_1|}
    \end{align}
    Note that 
    \begin{align}
        |z^Tu_1| = \bigg|\lambda_1 c_1^3 + \sum_i \lambda_i c_i^3 u_i^T u_1\bigg| \geq \lambda_{min}(|c_1|^3 - \frac{k\tau}{2\sqrt{2}}|c_1|^3) \geq 0.98 \lambda_{min} |c_1|^3.
    \end{align}
    Moreover,
    \begin{align}
        \|\Pperp(z)\|^2=\bigg\|\Pperp\bigg(\sum_{i\neq 1} \lambda_i (w^Tu_i)^3 u_i\bigg)\bigg\|^2\leq (1+\delta)\sum_{i\neq 1} \lambda_i^2 c_i^6 \leq \lambda_{max}^2 (1+\delta)\sum_{i\neq 1} c_i^6
    \end{align}
    Note that $c_i = c_1 u_1^T u_i + \Pperp(w)^T u_i$. Now we can use lemma~\ref{lemma: power6} to find an upper-bound for $c_i^6$.
    
    \begin{align}
        \|\Pperp(z)\|^2 \leq \lambda_{max}^2 (1.01 \sum_{i\neq 1} (u_i^T \Pperp(w))^6 + \underbrace{\bO(\sum_{i\neq 1}c_i^6\tau^6)}_{\leq c_1^6 \bO(k \tau^6)})
    \end{align}
    Note that $|u_i^T \Pperp(w)| \leq |c_i|+|c_1|\tau\leq |c_1|(\frac{1}{\sqrt{2}}+\tau)$ for $i\neq 1$. Thus, 
    \begin{align}
        \|\Pperp(z)\|^2 &\leq \lambda_{max}^2 c_1^4 (\frac{1}{\sqrt{2}}+\tau)^4 \underbrace{\sum_{i\neq 1}(u_i^T \Pperp(w))^2}_{\leq (1+\delta)\|\Pperp(w)\|^2} + \lambda_{max}^2 c_1^6 \bO(k\tau^6)\nonumber\\
        &\leq \lambda_{max}^2 2c_1^4 \max\bigg((\frac{1}{\sqrt{2}}+\tau)^4(1+\delta)\|\Pperp(w)\|^2, c_1^2 \bO(k\tau^6)\bigg)
    \end{align}
    Therefore,
    \begin{align}
        \frac{\|\Pperp(w)\|}{\underbrace{|w^Tu_1|}_{|c_1|}} &= \frac{\|\Pperp(z)\|}{|z^Tu_1|} \nonumber\\
        &\leq \frac{1.01}{0.98} \frac{\lambda_{max}}{\lambda_{min}}\max\bigg(\sqrt{2}\big(\frac{1}{\sqrt{2}}+\tau\big)^2\sqrt{1+\delta}{\frac{\|\Pperp(w)\|}{c_1}}%_{\frac{\|\Pperp(w)\|}{|w^Tu_1|}}
        , \bO(\sqrt{k}\tau^3)\bigg).
    \end{align}
    Now note that if $\kappa \leq \frac{5}{4}$, and $\delta$ is small enough, then:
    \begin{align}
        t = \frac{1.01}{0.98} \kappa \sqrt{2}\big(\frac{1}{\sqrt{2}}+\tau\big)^2\sqrt{1+\delta}< 1\nonumber
    \end{align}
    Thus,
    \begin{align}
        \frac{\|\Pperp(w)\|}{|w^Tu_1|} \leq \max \bigg(t\frac{\|\Pperp(w)\|}{|w^Tu_1|}, \bO(\kappa \sqrt{k}\tau^3)\bigg),
    \end{align}
    which completes the proof.
     
\end{proof}

\begin{theorem}
     If $k\tau\leq 0.05$ and $\kappa \leq \frac{5}{4}$ then for any local minimizer $w$ of \eqref{eq: opt1} there is an index $i$ such that $\|w-u_i\|\leq \bO(\kappa \sqrt{k}\tau^3)$.
\end{theorem}
\begin{proof}
    Based on the above lemma and without loss of generality assume that $\frac{\|\Pperp(w)\|}{|w^Tu_1|} \leq \bO(\kappa\sqrt{k}\tau^3)$. Thus, $\|\Pperp(w)\|^2 = 1-c_1^2 \leq c_1^2 \bO(\kappa^2 k\tau^6)$. Therefore,
    \begin{align}
        c_1 \geq \frac{1}{\sqrt{1+\bO(\kappa^2 k\tau^6)}} \geq 1-\bO(\kappa^2 k\tau^6)\nonumber.
    \end{align}
    Now we have $\|w-u_1\|^2 = 2(1-c_1)\leq \bO(\kappa^2 k\tau^6)$.
\end{proof}

% \vskip 0.2in
% \bibliography{sample}

\end{document}